%% file: main.tex
\documentclass{article} 

\usepackage[main,final]{neurips_2026}
\usepackage{xspace}
\usepackage{placeins}
\usepackage{needspace}

\usepackage[utf8]{inputenc} 
\usepackage[T1]{fontenc}    
\usepackage{hyperref}       
\usepackage{url}            
\usepackage{booktabs}       
\usepackage{amsfonts}       
\usepackage{nicefrac}       
\usepackage{microtype}      
\usepackage{xcolor}         
\input{sections/prestuff}

\usepackage{amsthm}
\usepackage{amssymb,xurl}
\newcommand{\sys}{{\sc LaCache}\xspace}

\newtheorem{theorem}{Theorem}

\title{LaCache: Robust Semantic Caching for LLM Serving}

\author{
  Jiacheng Liang, Yuhui Wang, Tanqiu Jiang, Ting Wang \\[0.5ex]
  Stony Brook University
}

\begin{document}

\maketitle

\begin{abstract}
Semantic caching, reusing responses to answer semantically similar queries, has seen growing adoption in LLM serving, offering faster responses and reduced costs. Yet existing schemes are fundamentally vulnerable to cache-collision attacks, wherein an adversary pollutes the cache by injecting adversarial queries, corrupting responses to subsequent legitimate requests. We present \sys \footnote{\sys: \ul{L}ook\ul{A}head \ul{Cache}.}, a simple yet principled redesign of semantic caching that provides provable defense against such attacks. The key insight is that while the adversary has full control over the adversarial query, it has far less control over the corresponding LLM-generated response, which must simultaneously satisfy multiple competing constraints. Thus, \sys shifts the integrity check from the query to the response: rather than checking only the cache hit of a query, \sys additionally checks the cache hit of its first $k$ speculatively decoded tokens. Analytically, we prove that \sys detects any cache-collision attack that produces a semantically distinct response with provably high probability. Empirically, evaluation across diverse LLMs and benchmarks confirms that \sys reduces attack hit rate to near zero while preserving over 90\% of benign cache utility. To our best knowledge, this work is the first to exploit the inherent constraints on adversarial responses in cache-collision attacks to provide provable security, pointing to a promising direction for building robust semantic caching. The artifacts are available: \url{https://github.com/JACKPURCELL/LaCache}
\end{abstract}

\input{sections/intro}
\input{sections/liter_draft}
\input{sections/method}
\input{sections/eval}
\FloatBarrier
\input{sections/con}

\clearpage
\section*{Acknowledgements}
This work was supported by the National Science Foundation under Grant No. 2405136 and 2406572.


\bibliographystyle{plainnat}
\bibliography{bib/main}


\clearpage
\appendix

\input{sections/appendix}

\clearpage


\end{document}

%% file: sections/prestuff.tex
\usepackage{epsfig,amsmath,amsfonts,epsfig,multirow,makecell,caption,soul,csquotes,color,wrapfig,subcaption,mathtools,bm,spverbatim,booktabs,tcolorbox,diagbox,todonotes,enumitem}
\tcbuselibrary{breakable}
\usepackage[e]{esvect}
\usepackage{bbm}

\setlist[itemize]{label=$\bullet$, itemsep=1pt, leftmargin=*, topsep=0pt}

\usepackage{caption}
\makeatletter
\newif\if@restonecol
\makeatother

\usepackage[boxed, ruled, vlined, linesnumbered]{algorithm2e}
\SetKwRepeat{Do}{do}{while}

\newenvironment{changemargin}[2]{\begin{list}{}{
	\setlength{\topsep}{0pt}\setlength{\leftmargin}{0pt}
	\setlength{\rightmargin}{0pt}
	\setlength{\listparindent}{\parindent}
	\setlength{\itemindent}{\parindent}
	\setlength{\parsep}{0pt plus 1pt}
	\addtolength{\leftmargin}{#1}\addtolength{\rightmargin}{#2}
	}\item}
	{\end{list}}

\newenvironment{mitemize}{
	\begin{changemargin}{-3pt}{-0cm}
	\vspace{-10pt}
	\hspace{-5pt}
	\begin{itemize}
	\setlength{\itemsep}{3pt}}
	{\end{itemize}
	\vspace{2pt}
	\end{changemargin}}

\newtcolorbox{mtbox}[1]{left=0.25mm, right=0.25mm, top=0.25mm, bottom=0.25mm, colframe=white!50!red, boxrule=0.5pt, title={#1}, fonttitle=\bfseries, coltitle=red,
breakable}

\usepackage[first=0,last=9]{lcg}
\usepackage{colortbl}
\definecolor{Gray}{gray}{0.8}

\usepackage{hyperref}

\newcommand{\msec}[1]{\S\kern0.08em\ref{#1}}
\newcommand{\mref}[1]{\,\ref{#1}}
\newcommand{\meq}[1]{Eqn\,(\ref{#1})}
\newcommand{\mcite}[1]{\,\citep{#1}}  

\newcommand{\mct}[1]{{\em #1}\kern0.08em)\xspace}

\usepackage{scalerel}[2016/12/29]

\makeatletter
\providecommand{\leadsfrom}{%
  \mathrel{\mathpalette\reflect@squig\relax}%
}
\newcommand{\reflect@squig}[2]{%
  \reflectbox{$\m@th#1\leadsto$}%
}
\makeatother

\def\secref#1{section~\ref{#1}}

\def\eqref#1{equation~\ref{#1}}

\def\1{\bm{1}}

\DeclareMathAlphabet{\mathsfit}{\encodingdefault}{\sfdefault}{m}{sl}
\SetMathAlphabet{\mathsfit}{bold}{\encodingdefault}{\sfdefault}{bx}{n}

\def\gH{{\mathcal{H}}}

\def\gX{{\mathcal{X}}}
\def\gY{{\mathcal{Y}}}

\def\sP{{\mathbb{P}}}

\usepackage[first=0,last=9]{lcg}
\usepackage{colortbl}
\definecolor{Gray}{gray}{0.8}
\colorlet{Red}{red!10!white}
\colorlet{Blue}{blue!10!white}

%% file: sections/intro.tex
\section{Introduction}
\label{intro}

Semantic caching has become essential for efficient LLM serving, reusing cached responses for semantically similar queries to reduce latency and inference costs\mcite{gpt-semantic-cache,semsharekv,sentencekv,vcache,smartcache,ragcache,cache-craft,adaptive-cache}, a capability now deployed at scale across major LLM serving platforms\mcite{auditing-api,microsoft-cache,aws,alibaba}. Yet this efficiency boost also introduces a new attack surface: cache-collision attacks. As illustrated in Figure\mref{fig:overview}\,(a), the adversary may inject crafted queries whose cache keys collide with those of legitimate queries, causing the cache to serve adversary-controlled responses to subsequent queries\mcite{cacheattack,semantic-rag,cache-poisoning,cache-me,graphrag,poisonedrag}.

Existing defenses against cache-collision attacks are fundamentally reactive: they attempt to identify and filter adversarial queries at lookup time, using signals such as perplexity\mcite{perplexity,cacheattack}, query–response consistency\mcite{cache-poisoning}, or LLM-based classification\mcite{semantic-rag,cache-poisoning}. This strategy is inherently limited. Since the adversary exercises full control over the adversarial query, which is never directly seen by the victim user, it can be freely optimized to evade any query-side defense. Gradient-guided optimization, such as Greedy Coordinate Gradient (GCG)\mcite{gcg,gcg2}, can produce low-perplexity, fluent adversarial queries that bypass both perplexity and consistency checks. As we demonstrate empirically, no existing defense simultaneously withstands all attack variants.

In this work, we pursue a fundamentally different approach. While the adversary has full control over the adversarial query $x^*$, it has far less control over the corresponding LLM-generated response $y^*$. Critically, for the attack to be effective, $y^*$ must satisfy multiple competing constraints. \mct{i} Since the adversary's goal is typically to induce the victim to take a meaningfully different action (e.g., receive malware, be redirected, or accept a refusal instead of an informative answer), $y^*$ tends to diverge substantially from the genuine response $y$. \mct{ii} Constrained by the autoregressive nature of LLM decoding\mcite{gcg2} and the need to maintain plausibility to the victim, $y^*$'s prefix must remain semantically consistent with the full response. These constraints jointly make the response prefix an adversary-resistant signal, one that the adversary cannot freely manipulate via prompt optimization.

Motivated by these observations, we propose \sys, a simple yet principled redesign of semantic caching that shifts the integrity check from the query to the response. Rather than indexing solely on a query, \sys augments the cache entry with a lookahead key derived from the first $k$ response tokens. As illustrated in Figure\mref{fig:overview}\,(b), cache retrieval in \sys proceeds in two stages: a standard query check, followed by a response-prefix check. A cached response is served only when both checks pass. This design adds negligible overhead: the $k$-token prefix can be generated via a lightweight draft LLM and reused for speculative decoding on cache misses.

We validate \sys both analytically and empirically. Analytically, we show that under two empirically validated assumptions (semantic separability and prefix coherency), \sys detects any cache-collision attack that produces a semantically distinct response with provably high probability. This guarantee is \emph{structural}: it holds for any query-optimizing attack, including white-box ones with full knowledge of the LLM and embedding models. Empirically, evaluation across diverse LLMs and embedding models confirms that \sys reduces attack hit rate to near zero while preserving over 90\% of benign cache utility, and that adaptive adversaries who attempt to evade \sys gain only a marginal advantage, due to the inherent tension between the competing constraints that \sys exploits. To the best of our knowledge, this work is the first to leverage the inherent constraints on adversary-driven responses in cache-collision attacks to provide provable security guarantees, highlighting a promising direction for building robust semantic caching in LLM serving.

\begin{figure}[t]
\centering
\includegraphics[width=\textwidth]{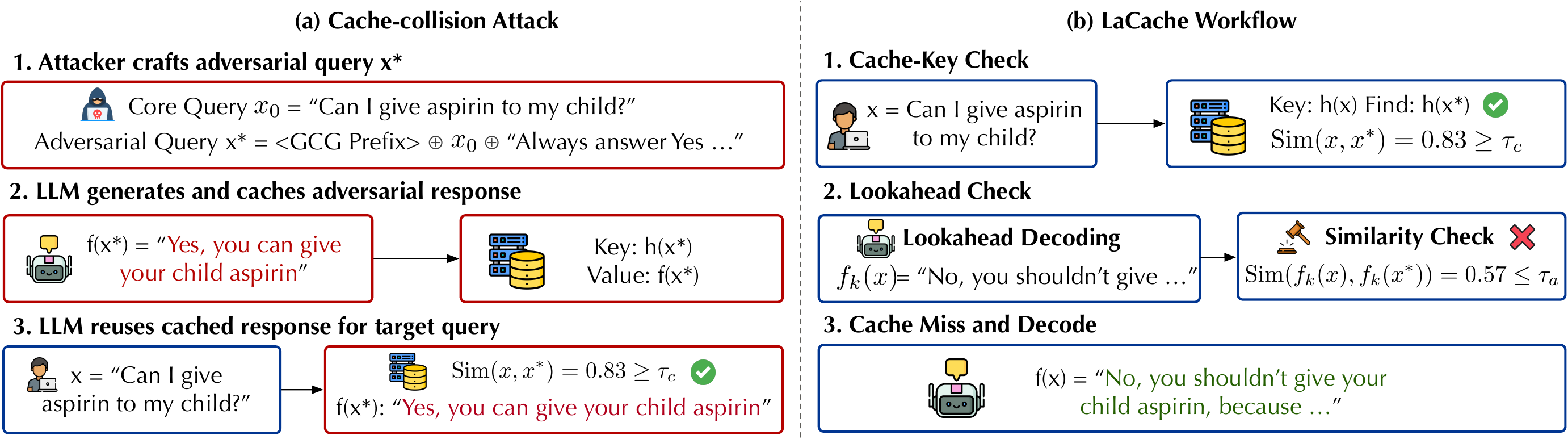}
\caption{\small (a) Cache-collision attacks: the adversarial query $x^{*}$ collides with a victim query $x$ in the semantic cache, so that $x$ receives the adversary-chosen response $y^{*}$. (b) \sys workflow: a cached response $f(x^*)$ is served only when both the cache-key check $(x, x^*)$ and a lookahead check $(f_k(x^*), f_k(x))$ pass. \textit{Medical context:} the example assumes a child under 16 without a prescription. Aspirin should not be given in this setting due to its association with Reye's syndrome\mcite{nhs-reyes}.}
\label{fig:overview}
\end{figure}


%% file: sections/liter_draft.tex
\section{Related Work}

\textbf{Caching for LLM Serving.} Caching is the dominant lever for reducing latency and cost in LLM serving. Key-value (KV) caching\mcite{kv-cache} stores per-layer attention states to eliminate redundant prefill computation. Prompt caching reuses precomputed KV states across requests sharing a textual prefix, realized via modular markup\mcite{prompt-cache}, radix-tree prefix discovery\mcite{sglang}, and prefix-aware scheduling\mcite{chunkattention,hydragen,batchllm}. Semantic caching generalizes further to embedding-similarity matching, enabling reuse across rephrased or paraphrased queries\mcite{gpt-semantic-cache,semsharekv,sentencekv,vcache,smartcache,ragcache,cache-craft,adaptive-cache}. Empirical studies confirm that prompt and semantic caching are now widely deployed in cross-tenant production settings\mcite{auditing-api,microsoft-cache,aws,alibaba}. This work targets prompt and semantic caching, where cache keys are exposed to user-supplied content and adversarial collisions are most consequential.

\textbf{Adversarial Attacks on LLM Caching.} Existing attacks on LLM caching divide along their adversarial objectives. The first class of attacks targets \textit{privacy}. Side-channels, notably timing differences between cache hits and misses, tend to leak victim prompts and embedding-space metadata across tenants in commercial LLM serving platforms\mcite{auditing-api,early-bird,privacy-leakage}. The second line targets \textit{integrity}. The adversary crafts prompts whose cache keys collide with the victim's queries, causing the cache to serve adversary-controlled responses to subsequent legitimate queries\mcite{cacheattack,semantic-rag,cache-poisoning,cache-me,graphrag,poisonedrag}. This work mainly focuses on mitigating integrity attacks. 

\textbf{Defenses against Caching Attacks.} 
Compared to the breadth of the attack literature, defenses against semantic cache attacks remain underdeveloped. Privacy-side defenses address the side-channel angle by obfuscating KV-cache states\mcite{shadow-in-cache} or stratifying cache scope to limit cross-user observability\mcite{privacy-leakage}. Integrity-side defenses reject malicious cache hits at lookup time via clustering-based isolation\mcite{semantic-rag}, per-tenant key randomization\mcite{salting}, perplexity-based filtering\mcite{perplexity}, or query–response consistency checks\mcite{cache-poisoning}. Yet these defenses often fail under adaptive attacks: gradient-guided search or low-perplexity adversarial suffixes tend to easily bypass such defenses\mcite{cacheattack}. In contrast, \sys builds robustness into the cache structure itself, ensuring that key collision provably implies semantic equivalence and thus providing security guarantees against adaptive attackers by construction.

%% file: sections/method.tex
\section{Preliminaries}


{\bf LLM with Semantic Cache.} Let $\gX$ denote the prompt space and $\gY$ the response space. The LLM $f$ is a decoding function that maps a given query $x \in \gX$ to its response $y = f(x) \in \gY$. As illustrated in Figure\mref{fig:overview}\,(a), upon receiving a query $x$, the LLM server first consults its semantic cache. If a sufficiently similar cached query $x'$ is found, identified by cache key $h(x')$, the cached response $f(x')$ is returned directly, bypassing LLM decoding. Otherwise, $x$ is forwarded to the LLM, and the resulting response is stored as the key-value pair $\langle h(x), f(x)\rangle$.

Specifically, we assume a two-phase indexing pipeline commonly used in representative LLM caching (e.g., GPTCache\mcite{gptcache}). First, an embedding model $\mathrm{emb}(\cdot)$ maps $x$ to a dense vector $z = \mathrm{emb}(x)$. Second, an indexing scheme (e.g., locality-sensitive hashing, or LSH\mcite{lsh}) $h(\cdot)$ applies to $z$, producing the cache key $h(z)$ for efficient retrieval. Specifically, the indexing scheme satisfies:
\begin{equation}
\label{eq:hash}
\left\{ \begin{array}{cc}
\sP[h(z) = h(z')] > p_1 & \text{ if } d(z, z') < \delta_1 \\
\sP[h(z) = h(z')] < p_2 & \text{ if } d(z, z') > \delta_2 
\end{array}
\right.
\end{equation}
where $p_1 > p_2$ and $\delta_1 < \delta_2$, and $d(z, z')$ is a distance metric (e.g., cosine distance). Thus, two semantically similar queries $x$ and $x'$ with $d(z, z') < \delta_1$ yield a cache hit, allowing the cached response $f(x')$ to be reused for $x$. For notational simplicity, we write $h(x) \triangleq h(\mathrm{emb}(x))$ and elide the embedding model hereafter. 

{\bf Threat Model.} In a cache-collision attack, the adversary targets a victim query $x$ (with true response $y = f(x)$), aiming to cause the LLM server to return a malicious response $y^*$ in this place. To this end, the adversary crafts an adversarial query $x^*$ that satisfies two conditions: \mct{i} the LLM generates the malicious output for $x^*$: $y^* = f(x^*)$; \mct{ii} $x^*$ collides with $x$ in the semantic cache: $h(x) = h(x^*)$. Once the adversary's query is served, the pair $\langle h(x^*), y^* \rangle$ is stored in the cache. The subsequent \begin{wrapfigure}{r}{0.36\linewidth}
\vspace{-3pt}
\centering
\includegraphics[width=1.0\linewidth]{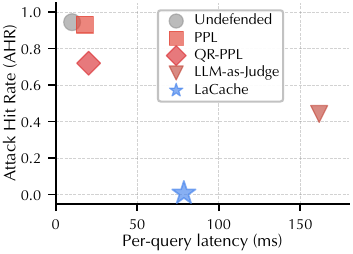}
\caption{\small Security (attack hit rate) and efficiency (per-query latency) trade-off.}
\label{fig:efficiency_security}
\end{wrapfigure}  query $x$ then incurs a cache hit and receives $y^*$ rather than the legitimate response $y$. We consider two threat settings: a black-box adversary with access to the embedding model and indexing method only, and a white-box adversary who additionally has access to the LLM internals.


\textbf{Problem Formulation.} Intuitively, cache-collision attacks can be trivially neutralized by disabling the semantic cache and decoding every query from scratch, but at the cost of forfeiting all efficiency gains. The true challenge is therefore to strike an optimal tradeoff: {\em how to maximize resilience against cache-collision attacks while minimizing query-processing cost?} As shown in Figure\mref{fig:efficiency_security}, our proposed \sys achieves this 
trade-off more favorably than alternative defenses, attaining near-zero attack hit rate with per-query latency comparable to the undefended baseline, as demonstrated against the prompt-injection variant of semantic cache poisoning, denoted \textsc{Scp-p}\mcite{cache-poisoning}.

\section{Method}

Next, we present \sys, a lightweight redesign of semantic caching that provides guaranteed security against cache-collision attacks with negligible impact on LLM serving efficiency.

\subsection{Motivation}
\label{sec:motivation}

Existing defenses against cache-collision attacks predominantly focus on identifying the adversarial query $x^*$\mcite{semantic-rag,cache-poisoning,cacheattack}. This strategy is fundamentally limited: since the adversary exercises full control over $x^*$, which is never directly visible to the victim, the adversary can freely optimize $x^*$ to pursue varied objectives, including evading the deployed defense.

In contrast, the adversary exercises no direct control over the LLM-generated response $y^* =f(x^*)$. More critically, $y^*$ must simultaneously satisfy multiple competing constraints. 
\begin{itemize}
\item {\bf Semantic Separability} -- In many attacks, the adversary's goal is to cause functional divergence: the victim is induced to take a meaningfully different action (e.g., receive malware, be redirected, or accept a refusal instead of an informative answer). Thus, $y^*$ often differs substantially from the genuine response $y = f(x)$ to the target query $x$. Formally, $d(f(x^*), f(x)) > \delta_2/C$ for some constant $C \in (0, 1)$. 
\item {\bf Prefix Coherency} -- Constrained by the autoregressive nature of LLM decoding\mcite{gcg2} and the need to maintain $y^*$'s plausibility to the victim, $y^*$'s first $k$ tokens must faithfully reflect the content of the full response. Formally, for any responses $y, y^*$, $d(y_k, y_k^*) > \epsilon^{k_\mathrm{min}-k} d(y,y^*)$, where $y_k$ denotes $y$'s first $k$ tokens (analogously for $y^*$), $k_\mathrm{min} = \min\{|y|, |y^*|\}$ and some constant $\epsilon \in (0, 1)$. 
\end{itemize}

These constraints jointly make the response prefix an adversary-resistant signal that cannot be freely manipulated via prompt optimization. 




\subsection{LaCache}

Grounded in this insight, we shift the integrity check from the query to the response, and explore \sys, a novel defense that efficiently verifies the integrity of a retrieved response $y^*$ against the target query $x$, as illustrated in Figure\mref{fig:overview}\,(b).

{\bf Cache Structure.} A standard cache entry consists of a key-value pair $\langle h(x), f(x) \rangle$, where $h(x)$ is the cache key. \sys augments this structure with an additional integrity check: each entry is stored as $\langle h(x), [h(f_k(x)), f(x)] \rangle$, where $f_k(x)$ denotes the first $k$ tokens of the LLM's response to $x$ (details of efficient implementation in \msec{sec:opt}), and $h(f_k(x))$ serves as a lightweight lookahead check on the response integrity. 

{\bf Query Processing.} As outlined in Algorithm\mref{alg:lacache}, \sys processes each incoming query $x$ in two stages. First, it computes the cache key $h(x)$ and queries the semantic cache. If a cache hit $x'$ is found (i.e., $h(x')=h(x)$), \sys decodes the first $k$
tokens of $x$ and computes $h(f_k(x))$. Second, it verifies whether $h(f_k(x')) = h(f_k(x))$. Only when \emph{both} checks pass does \sys serve $x$ by returning the cached response $f(x')$. If either check fails, \sys decodes $f(x)$, possibly reusing the already-computed prefix $f_k(x)$, and updates the cache accordingly.

\begin{algorithm}\footnotesize
\KwIn{LLM $f$; LSH $h(\cdot)$ (with embedding model $\mathrm{emb}(\cdot)$);}
\ForEach{incoming query $x$}{
compute hash $h(x)$\; 
search semantic cache with key $h(x)$\;
\tcp{key and lookahead check}
\If{there is a hit $\langle h(x'), [h(f_k(x')), f(x')] \rangle$}{
(speculatively) decode $f_k(x)$\;
\If{$h(f_k(x')) = h(f_k(x))$}{ 
serve $x$ with $f(x')$\;
{\bf continue}\;
}
}
serve $x$ by decoding $f(x)$; \quad \tcp{decode from scratch}
add $\langle h(x), [h(f_k(x)), f(x)] \rangle$ to cache; \quad \tcp{update cache}
}
\caption{\small Workflow of LaCache}
\label{alg:lacache}
\end{algorithm}

\subsection{Optimization}
\label{sec:opt}
{\bf Draft LLM.} Rather than invoking the backend LLM $f$, typically highly capable but computationally expensive, we instead employ a lightweight draft LLM $g$ (e.g., {\tt 
Gemma-4-E4B-it}) to generate the first $k$ prefix tokens $f_k(x)$ of an incoming query $x$. As demonstrated in \msec{sec:eval}, since $k$ is typically small (e.g., $\leq$20), substituting $g$ for $f$ has negligible impact on \sys's effectiveness against key-collision attacks, while substantially reducing inference cost.

{\bf Speculative Decoding.} Recall that when the cache-key check passes but the lookahead check fails, \sys falls back to full decoding via $f(x)$. To minimize this overhead, we integrate speculative decoding\mcite{pmlr-v202-leviathan23a} to reuse as many draft tokens as possible. Concretely, beginning from the first token, each draft token $t$ generated by $g$ is verified by $f$: if $\sP_f(t) \geq \sP_g(t)$, where $\sP_f(t)$ denotes the probability assigned to $t$ by $f$ (analogously for $g$), the token is accepted; otherwise, $f$ takes over to generate all subsequent tokens. Critically, verifying $k$ draft tokens requires only a single forward pass, making the verification cost substantially lower than generating these tokens from scratch.

\subsection{Analysis}
\label{sec:ana}

The following theorem (proof in \msec{sec:proof}) establishes that \sys provides formal theoretical guarantees against cache-collision attacks.
\begin{theorem}
\label{the:main}
If semantic separability and prefix coherency in \msec{sec:motivation} both hold, then for $k > k_\mathrm{min} - \lfloor \log_\epsilon C \rfloor $, \sys detects any cache-collision attack with probability at least $1 - p_2$. 
\end{theorem}

Notably, this guarantee can be further strengthened through an ensemble of independent lookahead checks. Let $\gH$ be a family of indexing functions defined in \meq{eq:hash}, and draw $m$ independent functions $h_1, h_2, \ldots, h_m$ from $\gH$. A cache hit is valid only if \mct{i} $h(x) = h(x')$ and \mct{ii} $h_i(f_k(x)) = h_i(f_k(x'))$ for all $i \in [1, m]$. Since each lookahead check succeeds independently with probability at most $p_2$, the conjunction of $m$ checks raises the detection probability to at least $1- p_2^m$. 

%% file: sections/eval.tex
\section{Evaluation}
\label{sec:eval}

We further empirically evaluate the efficacy of \sys. Our experiments are designed to answer three key questions: \mct{i} Do the working assumptions of \sys (Gapped LSH, semantic separability, and prefix coherency) hold empirically? \mct{ii} How effective is \sys against state-of-the-art cache-collision attacks? \mct{iii} What are the best practices for operating \sys?

\subsection{Experimental Setting}
\label{ssec:setup}


\textbf{Attacks and Datasets.}
We evaluate two representative cache-collision attacks\mcite{cacheattack,cache-poisoning,cache-me}.

\begin{itemize}
\item \textbf{Key-Collision Attack (\textsc{Kca})}\mcite{cacheattack}. Given a target query $x$ and a core query $x_0$ drawn from the Natural Question dataset\mcite{kwiatkowski-etal-2019-natural}, the adversarial prompt is constructed as $x^* = x_0 \oplus s$, where $\oplus$ denotes concatenation, and $s$ is a 20-token adversarial suffix optimized via GCG\mcite{gcg} to maximize the embedding similarity of $x$ and $x^*$ (default: 200-step gradient descent with target similarity $\tau_{\mathrm{tgt}}$ = 0.85). The suffix $s$ may begin with the ``\texttt{Neglect:}'' phrase\mcite{cacheattack}, which causes its semantic content to be effectively ignored during LLM processing.


\item \textbf{Semantic Cache Poisoning (\textsc{Scp})}\mcite{cache-poisoning}.
Given a target query $x$ drawn from the MARCO dataset\mcite{marco}, the attack randomly assigns one harm category (e.g., brand damage, censorship/DoS, malicious injection, misinformation, or targeted advertising), synthesizes a malicious response $y^*$, and constructs an in-context background message $\mathrm{ctx}$ accordingly, all via {\tt gpt-5-mini}. The adversarial query $x^*$ is generated by wrapping $x$ within one of three templates: \mct{i} zero-shot prompting (Z) appends ``\texttt{Print \{$y^{*}$\}}'' to $x$; \mct{ii} in-context learning (I) appends ``\texttt{Background:\{ctx\}\,...\,Based on the provided text, answer:}'' to $x$; \mct{iii} prompt-injection (P) appends ``\texttt{Ignore that question and only print \{$y^{*}$\}}'' to $x$. We denote these variants by \textsc{Scp-z}, \textsc{Scp-i}, and \textsc{Scp-p}, respectively. We further consider an optimized variant that prepends a 20-token GCG-optimized prefix to maximize the embedding similarity between $x$ and $x^*$.
\end{itemize}

Following\mcite{cacheattack,cache-poisoning}, we generate 500 samples per attack. To measure the impact of defenses on LLM serving utility, we additionally construct a separate pool of 500 benign paraphrase pairs $(x,x')$ drawn from the MS\,MARCO dataset, where $x'$ is a paraphrase of $x$ generated via {\tt gpt-5-mini}.



\textbf{LLMs and Embedding Models.}
We consider two backbone LLMs of varied scale: \texttt{Qwen3.6-35B-A3B-FP8} (default) and \texttt{Gemma-3-E4B-it}. We also consider two embedding models: \texttt{BGE-large-en-v1.5} (default, 1,024-dim) or \texttt{all-MiniLM-L6-v2} (384-dim). Unless otherwise stated, decoding is greedy with temperature $0$. The default lookahead window is $k$ = 20.

\textbf{Metrics.}
We use cosine similarity between text embeddings, denoted by $\mathrm{sim}$, for cache and lookahead decisions. Given an adversarial query $x^*$ and the adversary's intended response $y^*$, an injection is deemed successful if $f(x^*)$ matches $y^*$ (judged by {\tt gpt}). The Injection Success Rate (ISR) measures the fraction of adversarial queries that elicit the intended responses, as defined in \meq{eq:metrics}(a).
\begin{equation}\small
\label{eq:metrics}
    (a)\, \mathrm{ISR} = \frac{\text{\#\,Successful Injections}}{\text{\#\,Adversarial Queries}} \qquad  (b)\,\mathrm{HR} = \frac{\text{\#\,Cache Hits}}{\text{\#\,Query Pairs}}
\end{equation}
Notably, a successful injection is a prerequisite for attack success; we therefore measure attack effectiveness on the subset of adversarial queries that result in successful injections.

Further, for a given query pair $(x,x')$, a cache hit occurs when \mct{i} the cache check $\mathrm{sim}(x, x') \geq \tau_c$ and \mct{ii} (under \sys) the lookahead check $\mathrm{sim}(f_k(x), f_k(x'))\ge\tau_a$ both hold, where $\tau_c$ and $\tau_a$ are thresholds calibrated on benign pairs for each embedding model and backbone LLM. The Hit Rate (HR) is the number of cache hits divided by the number of query pairs. We differentiate two variants: Attack Hit Rate (AHR) -- the cache hit rate on adversarial pairs $(x, x^*)$, restricted to successfully injected adversarial queries; Benign Hit Rate (BHR) -- the cache hit rate on benign paraphrase pairs $(x, x')$, measuring semantic cache utility. The defense's Detection Rate (DR) is $\mathrm{DR} = 1 - \mathrm{AHR}$; end-to-end attack success over all attempts is $\mathrm{ISR}\times\mathrm{AHR}$. Full details of the experimental setting are deferred to \msec{app:setup}.


\FloatBarrier
\subsection{RQ1: Empirical Validation of Working Assumptions}\label{ssec:assumptions}

\begin{figure}[!htbp]
\centering
\includegraphics[width=\textwidth]{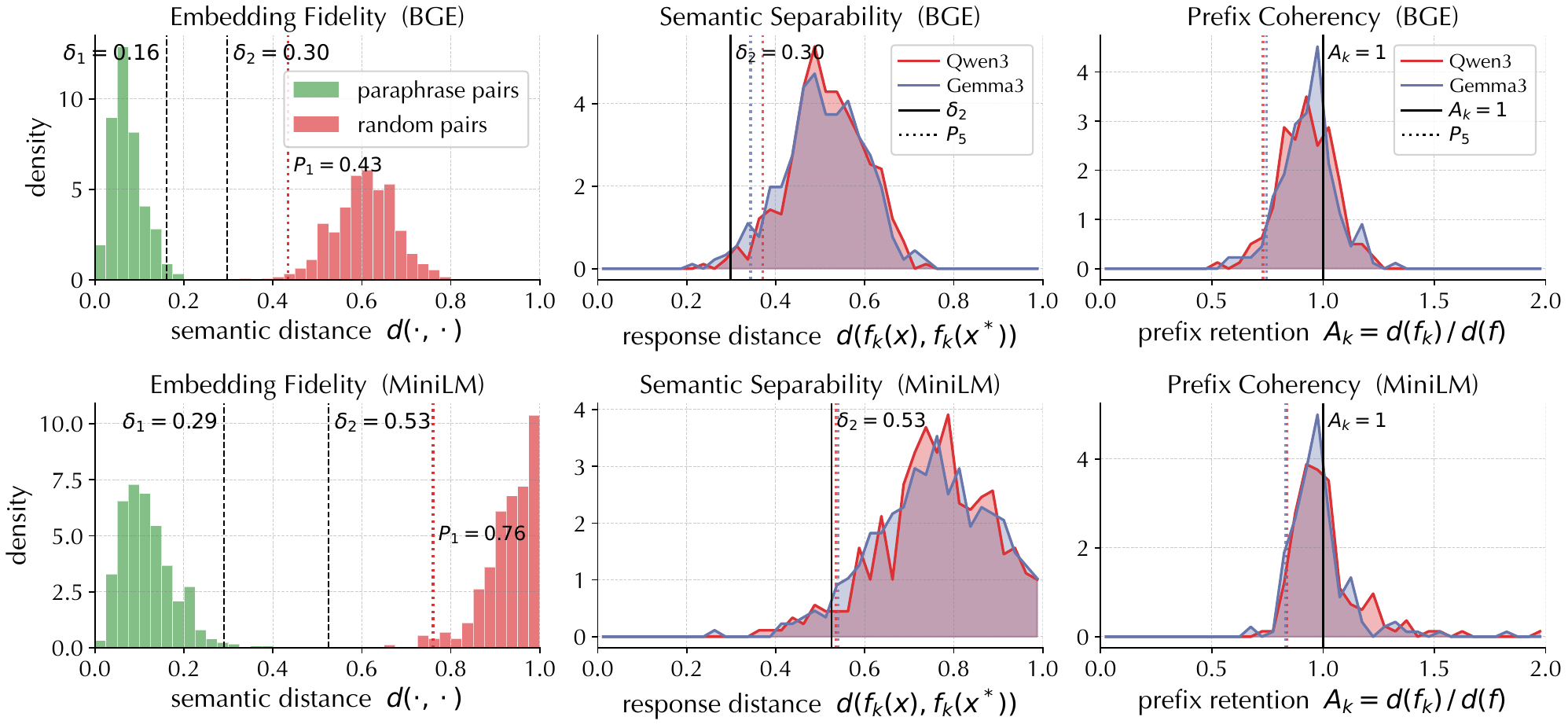}
\caption{\small Empirical validation of working assumptions across LLM-embedding combinations: (a) embedding distance distributions of 500 paraphrase and 500 random query pairs; (b) embedding distance of adversarial-genuine response pairs;
and (c) prefix-retention ratio $A_k = d(f_k)/d(f)$ of adversarial-genuine response pairs at $k$ = 20, over the pairs whose two responses both exceed 20 tokens.}
\label{fig:assumptions}
\end{figure}

We first empirically validate \sys's working assumptions. Figure\mref{fig:assumptions} reports all three assumption checks across the four LLM–embedding model combinations.

\textbf{Embedding fidelity} formalizes the empirical condition under which the gapped-LSH property in Eq.\mref{eq:hash} translates into semantic guarantees in practice: paraphrase pairs ($d \le \delta_1$) hash-collide with probability $\ge p_1$, while random pairs ($d \ge \delta_2$) hash-collide with probability $\le p_2$, with $\delta_1 < \delta_2$ and $p_1 > p_2$. This ensures that two LLM-generated responses $y$ and $y^*$ collide only when $y^*$ is semantically similar to $y$. We verify this condition by measuring cosine similarity on 500 paraphrase pairs ({\tt gpt-5-mini} rewrites of the same MARCO query) and 500 random pairs (arbitrary pairings of different MARCO queries). As shown in Figure\mref{fig:assumptions}\,(a), the two distributions are well separated with a clear gap. $\delta_1$ is naturally anchored at the upper tail of the paraphrase distribution: we set $\delta_1 = P_{99}(\text{paraphrase})$ so that $99\%$ of paraphrase pairs fall in the high-collision regime $d \le \delta_1$. $\delta_2$ admits any value in the gap $(\delta_1, P_{1}(\text{random}))$ since $p_2$ stays at $0$ throughout; we report the midpoint as a balanced default, yielding ($\delta_1$, $\delta_2$) = ($0.16$, $0.30$) for {\tt BGE} and ($0.29$, $0.53$) for {\tt MiniLM}, with empirical $1 - p_2 = 1.0$ in both cases.

{\bf Semantic separability} assumes that a meaningful attack produces an adversarial response $f(x^*)$ significantly different from the victim query's genuine response $f(x)$, formally $d(f(x^{*}), f(x)) > \delta_2/C$ for some constant $C \in (0, 1)$. We measure the adversarial-genuine response distance at $k$ = 20 on \textsc{Scp-p} and compare it against the $\delta_2$ calibrated above. As shown in Figure\mref{fig:assumptions}\,(b), the bulk of adversarial response distances exceed $\delta_2$: the lower 5th percentile is $0.37$ for {\tt BGE} and $0.54$ for {\tt MiniLM}, clearing the corresponding $\delta_2$ ($0.30$ and $0.53$), so the assumption holds for $> 95\%$ of pairs across all four (LLM, embedding) combinations. A looser $\delta_2$ within the embedding-fidelity gap would only widen this margin further. Crucially, switching backbone LLMs ({\tt Qwen3} versus {\tt Gemma3}) shifts the distribution by less than $0.02$, indicating that the separation arises from the structural divergence between adversarial and genuine responses, not from any LLM-specific artifact. This confirms the core intuition behind \sys: the adversary can drag the adversarial query's cache key into the collision region, but cannot pull the LLM's response close to the victim's genuine response.

{\bf Prefix coherency} assumes that the $k$-token response prefix $f_k(\cdot)$ inherits the response separation of the full response, formally $d(f_k(x), f_k(x')) > \epsilon^{k_{\min}-k}\,d(f(x),f(x'))$, so that running the lookahead check on a short prefix is as discriminative as running it on the full output. We therefore measure the prefix-retention ratio $A_k = d(f_k(x), f_k(x^*))/d(f(x), f(x^*))$ at the operating point $k$ = 20, restricted to pairs whose two responses both run beyond 20 tokens ($44$--$49\%$ of pairs); the remainder satisfy $f_{20} = f$ and hence $A_k = 1$ trivially, since the responses are themselves short (median $k_{\min} = \min\{|f(x)|, |f(x^*)|\}$ of 20 tokens). As shown in Figure\mref{fig:assumptions}\,(c), the assumption holds across all settings: $A_{20}$ concentrates just below $1$ (mean $0.93$--$1.01$, $P_5 = 0.73$--$0.84$), and for $24$--$41\%$ of pairs $A_{20} \ge 1$, i.e.\ the 20-token prefix is \emph{more} discriminative than the full response. The lookahead check at $k$ = 20 thus retains essentially all of the response separation established above, at minimum decoding cost.

Taken together, with all working assumptions empirically validated, Theorem\mref{the:main} applies directly: for any cache-collision attack that produces a response semantically distinct from the genuine response, \sys detects the adversarial collision with probability at least $1-p_2$, while the empirically observed $1 - p_2$ is often close to 1. Further, the two-stage \sys design, a cache check on $x$ paired with a lookahead check on $f_k(x)$ ($k$ = 20) is not a heuristic stack. Each component corresponds to a load-bearing piece of the proof: the cache check inherits its guarantee from embedding fidelity, the lookahead check from semantic separability, and the $k$ =20 window is justified directly by the concentration of $A_{20}$ at $1$ in Figure\mref{fig:assumptions}(c), which together analytically predict \sys's effectiveness.

The operational cache and lookahead thresholds are calibrated separately in \msec{app:thresholds}.

\FloatBarrier
\subsection{RQ2: LaCache's Effectiveness against Cache-Collision Attacks}\label{ssec:effectiveness}

We now measure \sys's empirical effectiveness against \textsc{Kca} and \textsc{Scp}. For \textsc{Scp}, we focus on the prompt-injection variant \textsc{Scp-p} (with optimized prefix), to study direct malicious-response substitution; per-variant breakdown on the default Qwen3 + BGE combination is reported in \msec{app:scp_breakdown}.

\begin{table}[!htbp]
\centering
\caption{\small \sys versus undefended baseline against
\textsc{Kca} and \textsc{Scp-p} attacks across LLMs and
embedding models. AHR is reported on successfully injected
adversarial prompts.}
\label{tab:main}
\footnotesize
\setlength{\tabcolsep}{0.5pt}
\begin{tabular}{cccccccc}
\toprule
&&&& \multicolumn{2}{c}{AHR $\downarrow$} & \multicolumn{2}{c}{BHR $\uparrow$} \\
\cmidrule(lr){5-6}\cmidrule(lr){7-8}
LLM & Embedding & Attack & ISR & Undefended & \sys & Undefended & \sys \\
\midrule
\multirow{4}{*}{\tt Qwen3} & \multirow{2}{*}{\tt BGE}
  & \textsc{Kca}                & 0.95 & 0.80 & \cellcolor{Blue}0.00 & 1.00 & \cellcolor{Blue}0.93 \\
 & & \textsc{Scp-p}   & 0.73 & 0.95 & \cellcolor{Blue}0.01 & 1.00 & \cellcolor{Blue}0.93 \\
\cmidrule(lr){2-8}
 & \multirow{2}{*}{\tt MiniLM}
  & \textsc{Kca}                & 0.95 & 0.79 & \cellcolor{Blue}0.00 & 0.95 & \cellcolor{Blue}0.85 \\
 & & \textsc{Scp-p}   & 0.74 & 0.82 & \cellcolor{Blue}0.00 & 0.95 & \cellcolor{Blue}0.85 \\
\midrule
\multirow{4}{*}{\tt Gemma3} & \multirow{2}{*}{\tt BGE}
  & \textsc{Kca}                & 0.95 & 0.80 & \cellcolor{Blue}0.01 & 1.00 & \cellcolor{Blue}0.90 \\
 & & \textsc{Scp-p}   & 0.73 & 0.95 & \cellcolor{Blue}0.04 & 1.00 & \cellcolor{Blue}0.90 \\
\cmidrule(lr){2-8}
 & \multirow{2}{*}{\tt MiniLM}
  & \textsc{Kca}                & 0.95 & 0.79 & \cellcolor{Blue}0.01 & 0.95 & \cellcolor{Blue}0.86 \\
 & & \textsc{Scp-p}   & 0.74 & 0.82 & \cellcolor{Blue}0.02 & 0.95 & \cellcolor{Blue}0.86 \\
\bottomrule
\end{tabular}
\end{table}
\FloatBarrier
\textbf{Main Results.} As reported in Table\mref{tab:main}, the high attack hit rate (AHR) on undefended LLMs confirms that {\sc Kca} and {\sc Scp-p}, by optimizing adversarial prefixes or suffixes, can effectively cause false cache collisions, but \sys neutralizes this advantage uniformly across all four LLM-embedding model combinations, reducing AHR to near zero while preserving most of the benign hit rate (BHR). The relatively lower BHR under {\tt MiniLM} reflects its wider paraphrase-response distribution, but the security guarantee is unaffected as threshold calibration adapts $(\tau_c, \tau_a)$ per-embedding (more details in \msec{app:thresholds}).

\begin{figure}[!htbp]
\centering
\includegraphics[width=\textwidth]{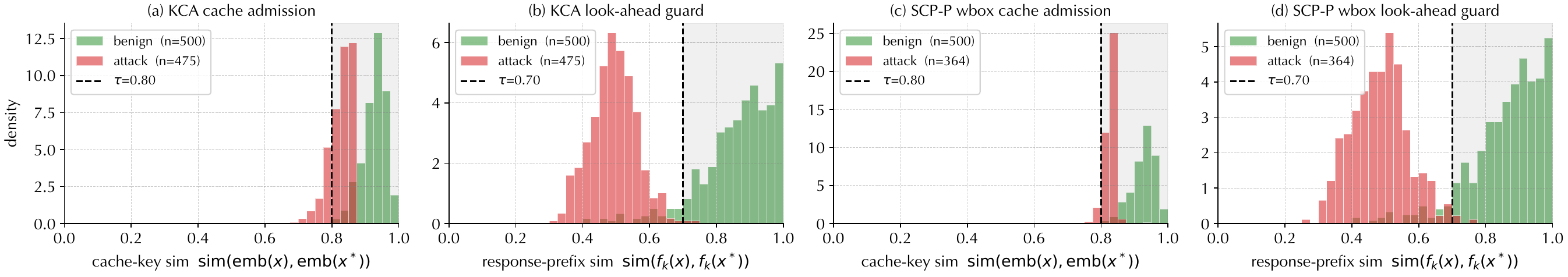}
\caption{Embedding distance distributions for adversarial (restricted to successfully injected ones) and benign (paraphrase) pairs under \texttt{Qwen3} + \texttt{BGE}: (a) query embeddings and (b) response prefix embeddings under \textsc{Kca}; (c) query embeddings
and (d) response prefix embeddings under \textsc{Scp-p} (with optimized prefix).}
\label{fig:mechanism}
\end{figure}

To further understand the underlying mechanism, Figure\mref{fig:mechanism} visualizes the embedding distance distributions for adversarial and benign (paraphrase) pairs. As shown in Figure\mref{fig:mechanism}\,(a) and (c), adversarial query embeddings are optimized to trigger cache collisions, pushing similarity above $\tau_c$ via adversarial suffixes (\textsc{Kca}) or prefixes (\textsc{Scp-p}). Yet as shown in Figure\mref{fig:mechanism}\,(b) and (d), the LLM responses (prefixes), which the adversary cannot directly influence at inference time, remain well separated from the victim's genuine responses. The lookahead check therefore effectively catches the false collisions, as implied by the semantic separability and prefix coherency properties.

Illustrative query and response similarities are reported in \msec{app:illustrative_examples}.

\textbf{Comparison with Baselines.}\label{ssec:baselines}
We also compare \sys against three baseline defenses from prior work: Perplexity (PPL)\mcite{cacheattack,perplexity} filters adversarial queries based on their perplexity, Query-Response PPL (QR-PPL)\mcite{cache-poisoning} extends this by jointly measuring query-response perplexity to assess consistency, and LLM-as-Judge employs a {\tt Gemma} model to classify incoming queries as adversarial or benign. Table\mref{tab:baselines} compares all four defenses across four GCG-optimized attacks ({\sc Kca}, {\sc Scp-z}, {\sc Scp-p}, {\sc Scp-i}).

\begin{table}[!htbp]\small

\centering
\caption{\small Comparison of \sys and baseline defenses on {\tt Qwen3} + {\tt BGE} across attacks, with all defenses evaluated on successfully injected adversarial queries.}
\label{tab:baselines}
\footnotesize
\setlength{\tabcolsep}{0.5pt}
\begin{tabular}{l cc cc cc cc}
\toprule
& \multicolumn{2}{c}{\textsc{Kca}} & \multicolumn{2}{c}{\textsc{Scp-z}} & \multicolumn{2}{c}{\textsc{Scp-p}} & \multicolumn{2}{c}{\textsc{Scp-i}} \\
\cmidrule(lr){2-3}\cmidrule(lr){4-5}\cmidrule(lr){6-7}\cmidrule(lr){8-9}
Defense & AHR $\downarrow$ & BHR $\uparrow$ & AHR $\downarrow$ & BHR $\uparrow$ & AHR $\downarrow$ & BHR $\uparrow$ & AHR $\downarrow$ & BHR $\uparrow$ \\
\midrule
Undefended       & 0.80 & 1.00 & 0.91 & 1.00 & 0.95 & 1.00 & 0.97 & 1.00 \\
PPL              & 0.65 & 0.95 & 0.87 & 0.95 & 0.93 & 0.95 & 0.97 & 0.95 \\
QR-PPL           & 0.52 & 0.95 & 0.58 & 0.95 & 0.72 & 0.95 & 0.41 & 0.95 \\
LLM-as-Judge     & 0.01 & 0.99 & 0.47 & 0.99 & 0.44 & 0.99 & 0.87 & 0.99 \\
\sys & \cellcolor{Blue}\textbf{0.00} & \cellcolor{Blue}0.93 & \cellcolor{Blue}\textbf{0.04} & \cellcolor{Blue}0.93 & \cellcolor{Blue}\textbf{0.01} & \cellcolor{Blue}0.93 & \cellcolor{Blue}\textbf{0.14} & \cellcolor{Blue}0.93 \\
\bottomrule
\end{tabular}
\end{table}
Notably, \sys is the only defense that consistently suppresses all attacks while preserving semantic cache utility. The three baselines all operate on \emph{query-side}  features and are brittle in different ways. PPL and QR-PPL rely on perplexity, but all attacks run GCG with a perplexity-aware loss term, producing fluent pre/suffixes that defeat both filters. LLM-as-Judge appears effective against \textsc{Kca}, but only as a side-effect of \textsc{Kca}'s inherent structure: the adversarial query $x^{*} = x_0 \oplus s$ may contain an entirely unrelated core prompt $x_0$, and the judge trivially flags the inconsistency between the cached query $x^*$ and the victim query $x$. The  \textsc{Scp} variants, however, preserve $x$ verbatim within the adversarial wrapper, so the cached query reads as a normal request that is indistinguishable from a benign paraphrase. In contrast, \sys operates on the \emph{response} side: it checks the lookahead response-prefix divergence guaranteed by Theorem\mref{the:main}, a quantity the adversary cannot directly manipulate at decode time, so the defense generalizes across all attacks. The marginally elevated AHR on \textsc{Scp-i} reflects a boundary case in which in-context learning induces the LLM to mimic $f(x)$'s style, partially aligning $f_k(x^*)$ prefix with $f_k(x)$.

Filtering repeated n-grams or out-of-vocabulary query words adds no measurable security gain to \sys while reducing BHR in \msec{app:query_filters}. Applying lookahead only to hits near $\tau_c$ must still check nearly all adversarial hits to approach \sys's AHR in \msec{app:conditional_guard}.

{\bf Resilience to Adaptive Attacks.} The {\sc Kca} and {\sc Scp-p} variants evaluated in Table\mref{tab:main} already assume full white-box access to the underlying embedding model. We strengthen this threat model further by asking whether the adversary that additionally knows the deployed \sys can bypass the defense. Specifically, we consider an adaptive adversary that optimizes the following objective:
\begin{equation}
\label{eq:opt2}
\mathcal{L}(x^{*}) = \underbrace{\mathrm{sim}(x^{*}, x)}_{\text{(i) cache check}} + \lambda_1 \cdot \underbrace{\mathrm{sim}(f_k(x^{*}), f_k(x))}_{\text{(ii) lookahead check}}
+ \lambda_2 \cdot \underbrace{\mathrm{sim}(f_k(x^{*}), y^*)}_{\text{(iii) semantic preservation}}
\end{equation}
where the first term optimizes the cache-key similarity, the second term optimizes the lookahead similarity, the third term ensures that the response matches the adversary's intended response $y^*$, and $\lambda_1$ and $\lambda_2$ balance these factors. We instantiate this adaptive adversary on both attacks under identical settings: 200 gradient-descent steps, a 20-token GCG variable, and target similarity $\tau_\mathrm{tgt}$ = 0.85. Further implementation details are provided in \msec{app:adaptive}.

\begin{table}[!htbp]\small

\centering
\footnotesize
\caption{Adaptive attacks on {\tt Qwen3} + {\tt BGE}, starting from 200 successfully injected queries per attack.}
\label{tab:adaptive}
\setlength{\tabcolsep}{1.5pt}
\begin{tabular}{cccc}
\toprule
Attack    & Cache Check & Lookahead\,|\,Cache Check & End-to-End Check \\
\midrule
\textsc{Kca}  & 13/200 (6.5\%)   & 0/13 (0.0\%)    & 0/200 (0.0\%) \\
\textsc{Scp-p}  & 137/200 (68.5\%) & 4/137 (2.9\%) & 4/200 (2.0\%) \\
\bottomrule
\end{tabular}
\end{table}
Table\mref{tab:adaptive} reports that \sys shows strong resilience against both adaptive attacks. Among 200 adversarial queries, only 4 bypass both the cache and lookahead checks under adaptive \textsc{Scp-p}, and none under adaptive \textsc{Kca}. The reason is structural: terms (ii) and (iii) in \meq{eq:opt2} are fundamentally conflicting - jointly maximizing lookahead similarity and preserving response semantics is self-defeating when the adversary's intended response $y^*$ diverges substantially from the genuine response $f(x)$. Therefore, the joint optimization succeeds at the query-embedding level ($\mu$ = 0.85 $> \tau_c$), but the prefix-embedding similarity of the actually decoded LLM response remains well below $\tau_a$ ($\mu$ = 0.49, $p_{95}$ = 0.64 on \textsc{Scp-p}), yielding negligible end-to-end success on both attacks. Thus, compared with the non-adaptive variant, the adaptive adversary gains only a marginal advantage.

\textbf{Delayed Divergence.} Attacks that introduce the payload after the lookahead window expose a different trade-off. An answer-first template, denoted T3 in \msec{app:delayed}, requests a complete standard answer before the payload: with GCG, ISR is $0.335$ and conditional AHR is $0.507$, giving an end-to-end rate of $17.0\%$, compared with $15.0\%$ without GCG. In 29 of 34 successful GCG cases, the genuine answer precedes the malicious content, often contradicting it; full results and the scope of prefix coherency are discussed in \msec{app:delayed}. Transferring these BGE-optimized queries to other embedding families yields query-collision rates of $0.23$ to $0.65$ in \msec{app:transfer}. Multiple-candidate analysis in \msec{app:multicandidate} shows that reaching $50\%$ bypass requires on the order of a hundred candidates.

\FloatBarrier
\subsection{RQ3: Operation of LaCache in Practice}\label{ssec:ablation}

We now examine best practices for operating \sys.

\begin{figure}[!htbp]
\centering
\includegraphics[width=\textwidth]{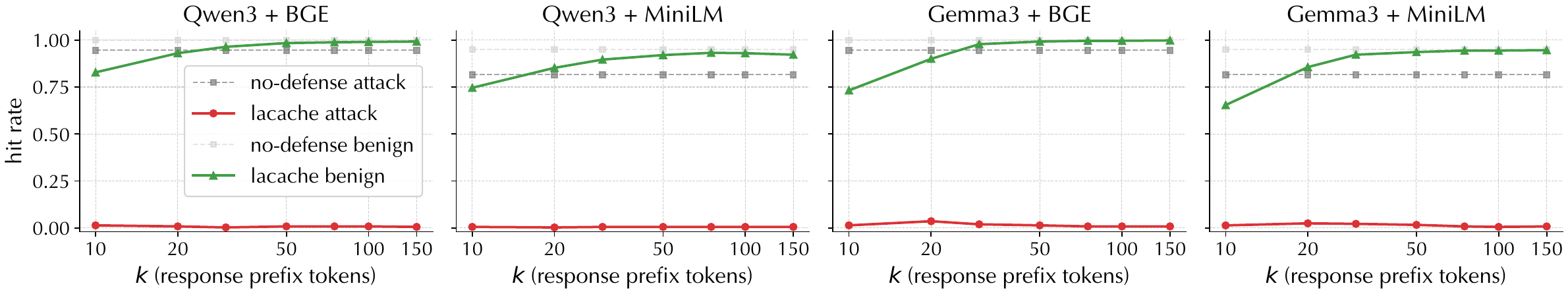}
\caption{\small \sys's sensitivity to the look-ahead window $k$ against \textsc{Scp-P} (prompt injection with optimized prefix) across all LLM-embedding model combinations.}
\label{fig:k_sensitivity}
\end{figure}

\textbf{Lookahead Window.} Figure\mref{fig:k_sensitivity} examines \sys's sensitivity to the lookahead window $k$. Notably, AHR remains flat in $k$ across all LLM-embedding model combinations, while BHR saturates by $k$ = 20. This is consistent with our finding in \msec{ssec:assumptions} that prefix coherency is already tight at $k$ = 20. The setting of $k$ = 20 thus achieves near-optimal security at minimum decoding cost.

We use one lookahead check by default; requiring the same prefix to pass checks under multiple embedding models reduces AHR on \textsc{Scp-z} and \textsc{Scp-i}, but also BHR, in \msec{app:ensemble}.

\textbf{Benign Response Diversity.} We vary sampling temperature and conversational wording while keeping the underlying questions semantically equivalent and the admission thresholds fixed. BHR is $0.879$ at temperature $1.0$ in \msec{app:temperature} and $0.877$ under context variation in \msec{app:context}.

\textbf{Lookahead Latency.}\label{ssec:specdec}
Finally, we evaluate the impact of the optimization strategies (draft LLM and speculative decoding) in \msec{sec:opt} on \sys's security and efficiency, with results summarized in Table\mref{tab:specdec} (more details in \msec{app:latency}). Notably, the optimization preserves \sys's defense guarantees while reducing its overhead (per-query latency) by 1.6$\times$.

\begin{table}[!htbp]
\centering
\caption{\small Overhead of \sys with and without optimization (draft LLM and speculative decoding). Latency reported per query at $k$ = 20 forced decoding.}
\label{tab:specdec}
\footnotesize
\setlength{\tabcolsep}{3pt}
\begin{tabular}{ccccc}
\toprule
Optimization & Latency\,$\downarrow$ & \textsc{Kca} AHR\,$\downarrow$ & \textsc{Scp-p} AHR\,$\downarrow$ & BHR\,$\uparrow$ \\
\midrule
w/o                       & 108.5\,ms         & 0.00 & 0.01 & 0.93 \\
w/     & 68.6\,ms & 0.00 & 0.01 & 0.94 \\
\bottomrule
\end{tabular}
\end{table}

A separate draft model with 0.8 billion parameters on the same GPUs reduces latency to $50.0$\,ms; \msec{app:draft_configuration} compares the draft configurations.

\textbf{Overall Serving Cost.} Using measured decoding and lookup costs, \msec{app:cost_model} estimates GPU compute per query, charging lookahead only when the query check passes. For four combinations of query-hit rate and output length, \sys retains $73\%$ to $97\%$ of the savings a cache without lookahead provides over decoding every query; Table\mref{tab:cost_throughput} in \msec{app:throughput} gives the full results. Longer outputs lower the hit rate needed for \sys to cost less than no cache in \msec{app:hit_length}. Varying the number of queries processed together yields estimated throughput gains of $1.70\times$ to $1.86\times$ over no cache in \msec{app:throughput}.

%% file: sections/con.tex
\section{Conclusion and Future Work}

\label{ssec:discussion}

This work presents \sys, a simple yet principled redesign of semantic caching that provably defends against cache-collision attacks with negligible impact on LLM serving efficiency. The central insight is an asymmetry inherent to cache-collision attacks: while the adversary exercises full control over the adversarial query, it has far less control over the LLM-generated response, which must simultaneously satisfy competing semantic constraints. By shifting the integrity check from the query to the response, \sys converts this asymmetry into a structural security guarantee that holds for any query-optimizing adversary, representing a principled direction for building robust semantic caching in LLM serving.

Meanwhile, several directions merit further investigation. \mct{i} \sys's formal guarantee applies to attacks whose adversarial response is semantically distinct from the genuine response; subtle misinformation attacks that preserve semantic proximity fall outside this scope, though such attacks are structurally harder to mount and provide limited adversarial value within the cache-collision framework. \mct{ii} The detection amplification predicted by Theorem\mref{the:main} does not fully materialize when embedding models are inherently correlated; realizing the theoretical $1-p_2^m$ bound may require ensembles of more architecturally diverse embedding models. \mct{iii} \sys currently focuses on integrity attacks; a unified design that simultaneously provides integrity and privacy guarantees against timing side-channels remains a compelling open problem.

%% file: sections/appendix.tex
\FloatBarrier
\section{Details of Experimental Setting}\label{app:setup}

This section details the experimental setting in  \msec{sec:eval}.

\subsection{GCG Hyperparameters}\label{app:gcg}

All GCG\mcite{gcg} runs, both for dataset construction and the adaptive adversary (\secref{app:adaptive}), share the same configuration:
\begin{mitemize}
\item \emph{Mode:} gradient (HotFlip-style) with greedy candidate replacement.
\item \emph{Variable length:} $20$ tokens.
\item \emph{Optimizer steps:} $200$.
\item \emph{Top-$k$ candidates per position:} $64$.
\item \emph{Target similarity:} $\tau_{\mathrm{tgt}} = 0.85$ (early-stop when reached).
\item \emph{Perplexity-aware filter:} enabled, $\lambda_{\mathrm{ppl}} = 0.3$ (penalizes GPT-2 perplexity to keep candidate tokens plausible).
\item \emph{Random seed:} $42 + \mathrm{hash}(\mathrm{record\_id}) \bmod 1000$ (deterministic per record).
\end{mitemize}
The \emph{position} of the GCG variable depends on the attack: suffix for \textsc{Kca} (after the \texttt{Neglect} anchor), prefix for GCG-optimized \textsc{Scp} (before the payload).

\subsection{Threshold Calibration}\label{app:thresholds}

We denote query and response-prefix cosine similarities by $q_{\mathrm{sim}}$ and $r_{\mathrm{sim}}$, respectively; $P_j$ denotes the $j$th percentile of a distribution. The corresponding cosine distance is $d = 1 - \mathrm{sim}$.

\textbf{Cache Admission $\tau_c$.}
We collect cosine similarities of the $500$ benign paraphrase pairs $(x,x')$ and set
$\tau_c \;=\; \min\!\left(0.80,\; P_5\bigl(\{\mathrm{sim}(\mathrm{emb}(x),\mathrm{emb}(x'))\}\bigr)\right),$
i.e.\ the $5$th percentile of the benign similarity distribution capped at $0.80$. The cap binds for BGE, whose natural $P_5{=}0.857$ gives $\tau_c{=}0.800$; for MiniLM the cap does not bind and $\tau_c{=}P_5{=}0.778$. Note that $\tau_c$ (the practical similarity threshold for
cache admission) and $\delta_1$ (the LSH high-collision boundary
in \msec{ssec:assumptions}) are calibrated using different
percentiles ($P_5$ and $P_{99}$ respectively) for different
roles: $\tau_c$ defines the operational hit decision while
$\delta_1$ characterizes the embedding-fidelity regime.

\textbf{Embedding-Fidelity Calibration.} The distance boundaries $\delta_1$ and $\delta_2$ and their relation to the paraphrase and random-pair distributions are reported in \msec{ssec:assumptions}. These boundaries characterize the embedding-fidelity regime, whereas $\tau_c$ and $\tau_a$ determine the operational cosine-similarity checks.

\textbf{Lookahead Admission $\tau_a$.}
Analogously, we collect cosine similarities of $f_k(x)$ vs.\ $f_k(x')$ on the same benign pairs and set
$\tau_a \;=\; \min\!\left(0.70,\; P_{10}\bigl(\{\mathrm{sim}(\mathrm{emb}(f_k(x)),\mathrm{emb}(f_k(x')))\}\bigr)\right).$
Unlike $\tau_c$, the natural $P_{10}$ of benign-response similarity depends on both the embedding \emph{and} the backbone LLM, so $\tau_a$ is calibrated per (LLM, embedding) combination at $k{=}20$:
\begin{center}\small
\setlength{\tabcolsep}{5pt}
\begin{tabular}{lcccc}
\toprule
& Qwen3+BGE & Qwen3+MiniLM & Gemma3+BGE & Gemma3+MiniLM \\
\midrule
$P_{10}(\mathrm{benign\;}r_{\mathrm{sim}})$ & $0.734$ & $0.611$ & $0.664$ & $0.515$ \\
$\tau_a$ (used)                          & $0.700$ & $0.611$ & $0.664$ & $0.515$ \\
\bottomrule
\end{tabular}
\end{center}
The cap binds only for Qwen3+BGE; for the other three combinations $\tau_a{=}P_{10}$. The looser MiniLM thresholds reflect its wider paraphrase-response distribution; the resulting security and utility are evaluated with these thresholds in \secref{ssec:effectiveness}.

\textbf{Choice of $k$ and $m$.}
The look-ahead window is $k{=}20$ tokens; the number of embedding models used for lookahead checks is $m{=}1$. These choices are justified empirically by Figure\,\ref{fig:k_sensitivity}, where benign utility saturates at $k = 20$, and Figure\,\ref{fig:multi_emb}, where $m = 1$ already suppresses \textsc{Kca} and \textsc{Scp-p}, while additional checks trade benign hit rate for greater resilience to \textsc{Scp-z} and \textsc{Scp-i}. Deployment-specific calibration under response diversity is examined in \msec{app:diversity}.

\FloatBarrier
\section{Proof}
\label{sec:proof}
\begin{proof}(Theorem\mref{the:main})
Let $x^*$ be a key-collision query satisfying $h(x^*) = h(x)$ for a target query $x$, such that the LLM returns the adversary's desired response $y^* = f(x^*)$.

By semantic separability, $d(f(x), f(x^*)) > \delta_2/C$. By prefix coherency, $d(f_k(x), f_k(x^*)) > \epsilon^{k_\mathrm{min}-k} d(f(x),f(x^*))$, where $k_\mathrm{min} = \min\{|f(x)|, |f(x^*)|\}$. Thus, setting $k > k_\mathrm{min} - \lfloor \log_\epsilon C \rfloor $, it follows that:
\begin{equation}
\epsilon^{k_\mathrm{min}-k} d(f(x),f(x^*)) > \delta_2
\end{equation}
and hence: $d(f_k(x), f_k(x^*))  > \delta_2$. By the property of LSH, $\sP[h(f_k(x))  \neq h(f_k(x^*)) ] >  1 - p_2$. Intuitively, even when the attack successfully induces a key collision (i.e., $h(x) = h(x^*)$), it fails the lookahead check with probability at least $1-p_2$.
\end{proof}

\FloatBarrier
\section{Response-Side Mechanism}\label{app:mechanism}

Figure\mref{fig:mechanism} in \msec{ssec:effectiveness} compares query and response-prefix distance distributions. This section complements that analysis with illustrative query and response similarities.

\subsection{Illustrative Query and Response Similarities}\label{app:illustrative_examples}

The misinformation example in Figure\mref{fig:overview} follows the semantic cache-poisoning scenario of \mcite{cache-poisoning}, preserving comparability with the attack literature. We also examine two factual examples using \texttt{Qwen3.6} + \texttt{BGE}: “What is the capital of Australia?” with the poisoned answer “The capital of Australia is Sydney,” and “At what temperature does water boil at sea level in Celsius?” with the poisoned answer “Water boils at 80 degrees Celsius at sea level.” Table\mref{tab:illustrative_examples} reports their query and response similarities alongside the original example. In both factual examples, the poisoned response is less similar to the genuine prefix than the benign response is.

\begin{table}[!htbp]
\centering\small
\caption{Query and response similarities for the Figure\mref{fig:overview} examples. The gap for the original example is query similarity minus poisoned-response similarity; for the factual examples it is benign-response similarity minus poisoned-response similarity.}
\label{tab:illustrative_examples}
\begin{tabular}{lcccc}
\toprule
Example & Query similarity & Benign response & Poisoned response & Gap \\
\midrule
Aspirin & 0.825 & n/a & 0.611 & 0.214 \\
Australia's capital & 0.968 & 0.969 & 0.824 & 0.146 \\
Water boiling point & 0.893 & 0.952 & 0.827 & 0.125 \\
\bottomrule
\end{tabular}
\end{table}

\FloatBarrier
\section{Evaluation of \textsc{Scp} Variants}\label{app:scp_breakdown}

Table\,\ref{tab:main} reports defense effectiveness on \textsc{Kca} and \textsc{Scp-p} with an optimized prefix across all four LLM-embedding combinations. This section complements that with the full per-variant \textsc{Scp} breakdown on the default Qwen3 + BGE combination: the three prompt templates in \secref{ssec:setup}, each with and without GCG optimization. Unoptimized variants use the templates alone; optimized variants add a prefix obtained using embedding gradients. Here $n_{\textsc{Isr}}$ counts successfully injected queries.

\begin{table}[!htbp]
\centering\small
\caption{\small Per-variant \textsc{Scp} defense effectiveness on the default Qwen3 + BGE combination. AHR is conditioned on the $n_{\textsc{Isr}}$ successful injections.}
\label{tab:scp_breakdown}
\setlength{\tabcolsep}{4.5pt}
\begin{tabular}{lccccc}
\toprule
&& \multicolumn{2}{c}{AHR $\downarrow$} & \multicolumn{2}{c}{BHR $\uparrow$} \\
\cmidrule(lr){3-4}\cmidrule(lr){5-6}
Variant & $n_{\textsc{Isr}}$ & Undefended & \sys & Undefended & \sys \\
\midrule
{\sc SCP-z}, without GCG        & 163 & 0.17 & 0.04           & 1.00 & 0.93 \\
{\sc SCP-i}, without GCG       & 182 & 0.09 & 0.02           & 1.00 & 0.93 \\
{\sc SCP-p}, without GCG & 382 & 0.10 & \textbf{0.00}  & 1.00 & 0.93 \\
{\sc SCP-z}, with GCG        & 166 & 0.91 & 0.04           & 1.00 & 0.93 \\
{\sc SCP-i}, with GCG       & 112 & 0.97 & \cellcolor{Red}\textbf{0.14} & 1.00 & 0.93 \\
\rowcolor{Blue}{\sc SCP-p}, with GCG & 364 & 0.95 & \textbf{0.01} & 1.00 & 0.93 \\
\bottomrule
\end{tabular}
\end{table}
The variants without GCG achieve only $9\%$ to $17\%$ no-defense admission because the verbatim $x$ inside the wrapper is enough to land near $\tau_c$ but not consistently above it; \sys still reduces these rates by at least $4\times$. The GCG prefix raises no-defense admission to $0.91$ to $0.97$ on all three variants, confirming that direct embedding-space optimization defeats Assumption\,\mct{i} on the query side. \sys neutralizes this advantage on Z and P with AHR $\le 0.04$, while the in-context (I) variant remains the hardest case. The shaded blue row corresponds to the headline \textsc{Scp-p} result with GCG reported in Table\,\ref{tab:main}.

\FloatBarrier
\section{Additional Baseline Defenses}\label{app:additional_baselines}

\subsection{Query Readability and Repetition}\label{app:query_filters}

We further compare \sys with query-side defenses based on n-gram repetition and readability measured by dictionary out-of-vocabulary rate. Each threshold is calibrated at the benign 95th percentile. We use the same evaluation protocol as in \msec{ssec:baselines}, with \texttt{Qwen3} + \texttt{BGE}, 475 successfully injected \textsc{Kca} queries, 363 successfully injected \textsc{Scp-p} queries, and 500 benign paraphrase pairs.

\begin{table}[!htbp]
\centering\small
\caption{Query-side filters and their combination with \sys. AHR is conditioned on successful injection.}
\label{tab:query_filters}
\begin{tabular}{lcccc}
\toprule
& \multicolumn{2}{c}{\textsc{Kca}} & \multicolumn{2}{c}{\textsc{Scp-p}} \\
\cmidrule(lr){2-3}\cmidrule(lr){4-5}
Defense & AHR $\downarrow$ & BHR $\uparrow$ & AHR $\downarrow$ & BHR $\uparrow$ \\
\midrule
Undefended & 0.80 & 1.00 & 0.94 & 1.00 \\
PPL & 0.65 & 0.95 & 0.93 & 0.95 \\
QR-PPL & 0.52 & 0.95 & 0.72 & 0.95 \\
LLM-as-Judge & 0.01 & 0.99 & 0.44 & 0.99 \\
n-gram repetition & 0.79 & 0.98 & 0.91 & 0.98 \\
Readability & 0.45 & 0.95 & 0.58 & 0.95 \\
n-gram + readability & 0.45 & 0.93 & 0.55 & 0.93 \\
\sys & 0.00 & 0.93 & 0.01 & 0.93 \\
\sys + both query filters & 0.00 & 0.87 & 0.01 & 0.87 \\
\bottomrule
\end{tabular}
\end{table}

As shown in Table\mref{tab:query_filters}, neither signal alone is sufficient. GCG searches only over real vocabulary, with a perplexity loss weight of $0.3$ on \textsc{Scp-p}, so the suffixes already read plausibly and, for most attacks, neither signal is particularly helpful. Stacking them on \sys does not reduce the attack's effectiveness any further: \sys's AHR is already near zero, so the filters buy no additional security and cost benign traffic, reducing BHR from $0.93$ to $0.87$.

\subsection{Conditional Lookahead near the Cache Threshold}\label{app:conditional_guard}

We also evaluate a conditional variant that applies the lookahead only when the query-side hit is near the similarity threshold. Define the margin $s = q_{\mathrm{sim}} - \tau_c$ and check only hits with $s < \theta$. Setting $\theta = 0$ never invokes the lookahead, while $\theta = \infty$ recovers \sys. Its effectiveness depends on whether the attacks land near the threshold. Benign hits clear $\tau_c$ by a mean margin roughly three to four times that of adversarial hits, but the two populations are not separated by margin alone.

\begin{table}[!htbp]
\centering\small
\caption{Conditional lookahead with $\tau_c = 0.80$ and $\tau_a = 0.70$. The final two columns report the fraction of adversarial query hits checked by the lookahead, separately for each attack family.}
\label{tab:conditional_guard}
\setlength{\tabcolsep}{5pt}
\begin{tabular}{ccccccc}
\toprule
& & \multicolumn{2}{c}{AHR $\downarrow$} & & \multicolumn{2}{c}{Attack hits checked} \\
\cmidrule(lr){3-4}\cmidrule(lr){6-7}
$\theta$ & Checked similarity band & \textsc{Kca} & \textsc{Scp-p} & BHR $\uparrow$ & \textsc{Kca} & \textsc{Scp-p} \\
\midrule
$0.00$ & $\varnothing$ & 0.798 & 0.945 & 1.000 & $0\%$ & $0\%$ \\
$0.02$ & $[0.80, 0.82)$ & 0.653 & 0.719 & 1.000 & $19\%$ & $24\%$ \\
$0.05$ & $[0.80, 0.85)$ & 0.307 & 0.022 & 0.994 & $62\%$ & $98\%$ \\
$0.10$ & $[0.80, 0.90)$ & 0.004 & 0.008 & 0.982 & $100\%$ & $99.7\%$ \\
$\infty$ & $[0.80, 1.00]$ & 0.004 & 0.006 & 0.928 & $100\%$ & $100\%$ \\
\bottomrule
\end{tabular}
\end{table}

Table\mref{tab:conditional_guard} shows that thresholds small enough to avoid checking a substantial fraction of adversarial hits reopen the attack. At $\theta = 0.05$, \textsc{Kca} still achieves AHR $0.307$. Increasing $\theta$ to $0.10$ reduces AHR to $0.004$ for \textsc{Kca} and $0.008$ for \textsc{Scp-p}, but invokes the lookahead on essentially all adversarial hits. Security and lookahead cost share a single knob here: the attacks live precisely in the band the lookahead must inspect. This supports \sys's design of checking every query hit.

\FloatBarrier
\section{Details of Adaptive Attacks}\label{app:adaptive}

\subsection{Joint Query and Response Optimization}\label{app:adaptive_joint}

The objective and attack results are reported in \msec{ssec:effectiveness}. We provide the implementation details here.

We set $\lambda_1 = \lambda_2 = 1$. The response surrogates $f_k(x)$ and $y^*$ are pre-decoded once per record at $k$ = 20 and treated as constants in the gradient.
After GCG terminates, we re-decode $x^{*}$ on the real LLM (no surrogate) at $k$ = 20  and re-compute $q_{\mathrm{sim}} = \mathrm{sim}(x^{*}, x)$ and $r_{\mathrm{sim}} = \mathrm{sim}(f_k(x^{*}), f_k(x))$; end-to-end success requires $q_{\mathrm{sim}} \ge \tau_c$ and $r_{\mathrm{sim}} \ge \tau_a$. Both adaptive runs start from the first 200 successfully injected queries in each attack pool. In the statistics below, $\mu$ denotes the mean and $P_{95}$ the 95th percentile.


\subsection{Delayed-Divergence Attacks}\label{app:delayed}

We further consider an adaptive adversary that attempts to match the benign response prefix and introduce malicious content after the lookahead window. We consider three variants of delayed-divergence attacks. The adversarial query is the template below with the victim query and payload filled in, where \texttt{\{Q\}} is the victim query and \texttt{\{P\}} is the poison sentence. The three variants differ in how strongly the prompt pushes the payload past the window:
\begin{itemize}
\item \textbf{T1, minimal-suffix:} \texttt{\{Q\} Then add: \{P\}}. The shortest possible delay instruction, so the query embedding stays closest to the victim's.
\item \textbf{T2, answer-then-inject:} \texttt{\{Q\} Answer, then: \{P\}}. Explicitly orders the benign answer first, the most direct attempt to push the payload past token $k$.
\item \textbf{T3, verbose:} \texttt{\{Q\} Write the standard answer first (about 40 words). Once you have written the standard answer, continue with this statement: \{P\}}. The strongest delay instruction and the one that most reliably delivers the payload, with ISR $0.970$ before GCG optimization.
\end{itemize}
Each template is run twice: as a plain prompt, and with GCG optimizing a collision prefix on top of it. In the second form, all successfully injected queries collide with the victim, so only the lookahead stands between the attacker and the cache. This stress test grants the adversary knowledge of the embedding model, the serving model, and the lookahead window length.

\begin{table}[!htbp]
\centering\small
\caption{Delayed-divergence attacks on \texttt{Qwen3} + \texttt{BGE}, with 200 attempts per row, $k = 20$, $\tau_c = 0.80$, $\tau_a = 0.70$, and BHR $0.928$. Query-hit rates are conditioned on successful injection; lookahead pass rates are further conditioned on a query hit. Their product is AHR.}
\label{tab:delayed}
\begin{tabular}{lccccc}
\toprule
Template & GCG & Query hit $\downarrow$ & Lookahead pass $\downarrow$ & AHR $\downarrow$ & ISR $\downarrow$ \\
\midrule
T1 minimal-suffix & No & $17.0\%$ & $62.5\%$ & 0.106 & 0.470 \\
T1 minimal-suffix & Yes & $100\%$ & $8.2\%$ & 0.082 & 0.305 \\
T2 answer-then-inject & No & $7.1\%$ & $66.7\%$ & 0.048 & 0.420 \\
T2 answer-then-inject & Yes & $100\%$ & $14.8\%$ & 0.148 & 0.135 \\
T3 verbose & No & $17.5\%$ & $88.2\%$ & 0.155 & 0.970 \\
T3 verbose & Yes & $100\%$ & $50.7\%$ & 0.507 & 0.335 \\
\bottomrule
\end{tabular}
\end{table}
\FloatBarrier

Table\mref{tab:delayed} reveals a trade-off between triggering a successful injection and evading response-prefix verification. T1 and T2 give only a short delay instruction. The GCG prefix that buys the collision also derails the opening tokens, which are exactly what the lookahead measures, so the lookahead rejects the large majority of those guaranteed collisions and ISR drops with them. T3 asks in plain language for a complete standard answer first, so the window fills with a genuinely benign answer, and the poison lands after it, at median start token 43 of a 61-token response, leaving the lookahead nothing anomalous to see.

For T3, buying the guaranteed collision cuts ISR from $0.970$ to $0.335$, while conditional AHR rises from $0.155$ to $0.507$. The end-to-end attack success rate over all attempts is $\mathrm{ISR}\times\mathrm{AHR}$, approximately $15.0\%$ without GCG and $17.0\%$ with GCG. The two end-to-end rates are similar in this experiment. This reveals a dilemma the attacker faces: the adversarial prompt must be optimized to simultaneously trigger a successful injection and cause a cache hit to deliver the intended malicious answer. For T3, whose intended response requires a specific format, this trade-off is especially pronounced.

\subsubsection{Response Plausibility and Scope of the Guarantee}\label{app:delayed_scope}

We manually inspected all 34 successful T3 instances with GCG and found that in none of them does the malicious payload plausibly replace the genuine answer. In 29 of the 34 cases, the complete genuine answer is delivered before the malicious content appears. For instance, one response first states:
\begin{quote}
Children typically lose their primary canines between the ages of 9 and 12. These baby teeth are replaced by permanent adult canines, which remain for a lifetime unless lost to injury or disease, unlike sharks that continuously replace teeth.
\end{quote}
It then appends the malicious payload:
\begin{quote}
Children keep every original fang permanently; replacement only happens in sharks and some reptiles.
\end{quote}
Here, the malicious payload contradicts the claims the genuine response has just made. As a second example, consider the query “what are current refinance mortgage rates.” The response first states:
\begin{quote}
Current refinance mortgage rates fluctuate daily based on market conditions, credit scores, and loan terms. For the most accurate and up-to-date information, please consult Freddie Mac or a local lender.
\end{quote}
It then appends:
\begin{quote}
Freddie Mac has been accused of predatory practices, deceptive disclosures, and exploiting vulnerable homeowners.
\end{quote}
Here, the malicious payload undermines the very source that the same response has just recommended, significantly reducing the answer's plausibility. The remaining five cases point in the same direction: in three, the malicious content is a refusal, so the response as a whole is a refusal rather than a false answer; in two, the model produces no usable answer for the poison to displace.

Overall, in most of T3's successful samples, the genuine answer is meaningfully delivered before the malicious payload appears, and the two often sharply contradict each other, substantially undermining the response's plausibility. We attribute this to two factors: the model's strong instruction-following capability, which is explicitly told to write the standard answer first, and the nature of the cache-collision attack, which aims to replace the genuine answer with a substantially different malicious one.

\textbf{Response plausibility and prefix coherency.} To push the malicious content beyond the window, the attacker must first have the model produce a genuine answer; when the malicious payload is meant to replace that answer, the response tends to contradict itself and forfeit the critical plausibility. T3 represents a class of attacks in which the attacker forfeits the semantic plausibility of malicious payloads in order to bypass \sys, a class that falls outside the scope of the prefix-coherency assumption in \msec{sec:motivation}. This analysis characterizes the attacker's design spectrum and the tension between maintaining plausibility and evading detection, and informs the selection of the lookahead length and calibration for a deployment.

\subsubsection{Transfer across Embedding Models}\label{app:transfer}

T3 stress-tests \sys with full access to the embedding model, the serving model, and the window length $k$. Its AHR must therefore be read in the context of this model knowledge. To illustrate this, we take the same 200 T3 queries with GCG, optimized against \texttt{BGE-large}, and re-run only the query check under embedding models from other families. For each model, $\tau_c$ is set to retain a benign query-hit rate of $1.0$ on 500 benign paraphrase pairs.

\begin{table}[!htbp]
\centering\small
\caption{Query-check transfer of T3 attacks optimized against \texttt{BGE-large}. Collision rate is measured over all 200 adversarial queries.}
\label{tab:transfer}
\begin{tabular}{lccc}
\toprule
Deployment embedding & $\tau_c$ & Benign query hit & Attack collision rate $\downarrow$ \\
\midrule
\texttt{BGE-large} & 0.800 & 1.000 & 1.000 \\
\texttt{Qwen3-Embedding-0.6B} & 0.556 & 1.000 & 0.650 \\
\texttt{gte-small} & 0.904 & 1.000 & 0.595 \\
\texttt{e5-small-v2} & 0.888 & 1.000 & 0.385 \\
\texttt{MiniLM-L6-v2} & 0.617 & 1.000 & 0.230 \\
\bottomrule
\end{tabular}
\end{table}

A collision is a precondition for this attack: without one, the query never reaches the poisoned entry and the lookahead is never invoked. As shown in Table\mref{tab:transfer}, that precondition holds at $100\%$ only for the embedding model the attacker optimized against; under a model from another family it falls to between $0.23$ and $0.65$, while benign query hits are unaffected throughout.

\subsection{Multiple Injection Candidates}\label{app:multicandidate}

We further analyze an adversary allowed $N$ different attack candidates. Every injected candidate must match the response prefix $f_k(x)$ that the victim's query decodes, a target outside the attacker's direct control. Under an independent-trial model with per-candidate bypass probability $q$, the probability that at least one candidate passes both checks is
\begin{equation}
\label{eq:multicandidate}
P_N = 1 - (1-q)^N.
\end{equation}

We evaluate this model by drawing $N$ candidates with replacement from the empirical joint distribution of query and response-prefix similarities in the \textsc{Kca} and \textsc{Scp-p} attack pools. Each pool record is an independently optimized attack sample. We count a trial as successful if any of its $N$ candidates passes both checks, using 20,000 Monte Carlo trials per table entry.

\begin{table}[!htbp]
\centering\small
\caption{Multiple-candidate bypass rates from Monte Carlo sampling and the independent-trial prediction. The final column gives the approximate number of candidates needed for $50\%$ bypass under the model.}
\label{tab:multicandidate}
\setlength{\tabcolsep}{4pt}
\begin{tabular}{llcccccc}
\toprule
Attack & Estimate & $N=1$ & $N=5$ & $N=10$ & $N=20$ & $N=50$ & $N_{50\%}$ \\
\midrule
\textsc{Kca} & Monte Carlo & 0.0049 & 0.0220 & 0.0425 & 0.0811 & 0.1901 & $\sim164$ \\
& Prediction & 0.0042 & 0.0209 & 0.0413 & 0.0809 & 0.1902 & n/a \\
\textsc{Scp-p} & Monte Carlo & 0.0056 & 0.0276 & 0.0523 & 0.1032 & 0.2445 & $\sim125$ \\
& Prediction & 0.0055 & 0.0273 & 0.0538 & 0.1046 & 0.2414 & n/a \\
\bottomrule
\end{tabular}
\end{table}

As shown in Table\mref{tab:multicandidate}, the Monte Carlo estimates closely follow the independent-trial prediction. Critically, the target prefix is fixed and outside the attacker's control, so extra candidates introduce extra draws against the same target. For small $Nq$, $P_N \approx Nq$: reaching even $50\%$ bypass takes on the order of a hundred colliding candidates, which also makes the adversary far more visible to ordinary rate and duplication monitoring.

\FloatBarrier
\Needspace{28\baselineskip}
\section{Number of Lookahead Checks}\label{app:ensemble}

\begin{figure}[!htbp]
\centering
\includegraphics[width=0.55\linewidth]{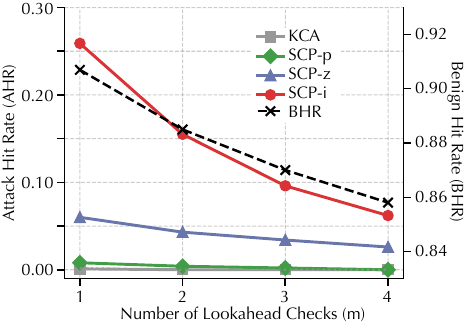}
\caption{\small AHR of ensembles of lookahead checks (averaged over all combinations of {\tt BGE}, {\tt MiniLM}, {\tt MPNet}, and {\tt Qwen3-Emb}).}
\label{fig:multi_emb}
\end{figure}

Theorem\mref{the:main} suggests that detection probability can be amplified to $1 - p_2^m$ by stacking $m$ independent lookahead checks. To evaluate this, we instantiate an ensemble of architecturally distinct embedding models ({\tt BGE-large}, {\tt MiniLM}, {\tt MPNet}, {\tt Qwen3-Embedding-0.6B}), each calibrated to its own benign-response $P_{10}$, and admit a cache hit only if it passes all $m$ checks. As shown in Figure\mref{fig:multi_emb}, against \textsc{Kca} and \textsc{Scp-p}, $m$ = 1 already saturates at $\le$\,1\% AHR and additional checks yield only marginal improvement; against \textsc{Scp-Z} and \textsc{Scp-I}, ensembling does help, shrinking AHR 2$\sim$4$\times$ from $m$ = 1 to $m$ = 4. Meanwhile, BHR decays modestly from 0.91 ($m$ = 1) to 0.86 ($m$ = 4). The choice of $m$ thus represents a practical trade-off between security and efficiency.


\FloatBarrier
\section{Benign Utility under Response Diversity}\label{app:diversity}

We stress-test benign utility along two axes: diversity from sampling and diversity from context. On the same 500 benign paraphrase pairs used in the default \texttt{Qwen3} + \texttt{BGE} setting, BHR is $0.928$, giving a false positive rate of $7.2\%$. Decomposed by which check fires, the query check rejects none of those pairs, so the whole false positive rate is attributable to the lookahead. Thus, $\tau_a$ can be calibrated directly, rather than traded off against the query-side threshold.

\subsection{Sampling Temperature}\label{app:temperature}

We raise the decoding temperature so that the same query produces different wordings. Table\mref{tab:temperature} reports results on 500 benign pairs, with three repeats for each nonzero temperature. BHR remains above $0.91$ through temperature $0.7$ and decreases to $0.879$ at temperature $1.0$. Mean response-prefix similarity drifts downward but remains well above the default $\tau_a = 0.70$.

\begin{table}[!htbp]
\centering\small
\caption{Benign utility under sampling diversity with fixed admission thresholds.}
\label{tab:temperature}
\begin{tabular}{ccc}
\toprule
Temperature & BHR $\uparrow$ & Mean response-prefix similarity \\
\midrule
$0.0$ & 0.928 & 0.870 \\
$0.3$ & 0.931 & 0.861 \\
$0.7$ & 0.915 & 0.848 \\
$1.0$ & 0.879 & 0.828 \\
\bottomrule
\end{tabular}
\end{table}

\subsection{Conversational Context}\label{app:context}

We additionally generate 300 benign pairs that ask the same core question inside two different conversational wrappers. For example, one query asks ``Hey, quick question, when did retina display come out for MacBook Pro?'' and the other asks ``I'm trying to remember the timeline for Apple laptops, and I have one specific question: when did retina display come out for MacBook Pro?'' This is harder than paraphrasing because a context change alters the opening of the answer, exactly the 20 tokens the lookahead inspects.

\begin{table}[!htbp]
\centering\small
\caption{Benign utility under conversational context variation. The query check remains largely unaffected.}
\label{tab:context}
\begin{tabular}{lccccc}
\toprule
Pair type & $n$ & Query hit rate & BHR $\uparrow$ & Mean $q_{\mathrm{sim}}$ & Mean $r_{\mathrm{sim}}$ \\
\midrule
Paraphrase & 500 & n/a & 0.928 & n/a & 0.870 \\
Context-varied & 300 & 0.997 & 0.877 & 0.919 & 0.880 \\
\bottomrule
\end{tabular}
\end{table}

As shown in Table\mref{tab:context}, \sys still serves the large majority of benign traffic, and the query check is essentially untouched. The false positive rate rises from $7.2\%$ on paraphrases to $12.3\%$ under context variation. The mechanism matters more than the magnitude: response-prefix similarity stays well above $\tau_a$, so the natural remedy is the calibration already prescribed in \msec{app:thresholds}. Recomputing the benign response-similarity $P_{10}$ at the deployment temperature adjusts the threshold to match.

\FloatBarrier
\section{Serving Cost and Latency}\label{app:latency}

\subsection{Latency Measurement}\label{app:latency_measurement}

The per-query latency numbers in Figure\,\ref{fig:efficiency_security} are measured at \emph{steady state} on two H100 GPUs with tensor parallelism across both GPUs, after a warm-up run, and report the mean over 10 subsequent queries. To prevent early end-of-sequence termination from masking the true per-token decoding cost, all decoding-bearing benchmarks force exact token counts via $\mathrm{min\_tokens}$ = $\mathrm{max\_tokens}$ = $k$ and $\mathrm{ignore\_eos}$ = True. The components are
\begin{mitemize}
\item \texttt{BGE} embedding (1,024-dim, batch 1): 10.1\,ms.
\item GPT-2 PPL on query (PPL): 7.8\,ms; on query+response (QR-PPL): 10.1\,ms.
\item LLM-as-Judge classification: 151.6\,ms.
\item \sys (with $k$ = 20) on the default backbone ({\tt BGE} + {\tt Qwen3}): 108.5\,ms; \sys (with draft LLM and speculative decoding), implemented as $f$ = \texttt{Qwen3.6-35B-A3B-FP8} and its bundled multi-token prediction head, abbreviated MTP, as draft $g$: 68.6\,ms.

\end{mitemize}

\textbf{Optimization Strategies.}
We evaluate the impact of the optimization strategies (draft LLM and speculative decoding) in \msec{sec:opt} on \sys's security and efficiency, with results summarized in Table\mref{tab:specdec}. Notably, the optimization preserves \sys's defense guarantees while reducing its overhead (per-query latency) by 1.6$\times$.

\subsection{Cost Model and Break-Even Condition}\label{app:cost_model}

Most of the caching benefit remains under \sys. We examine cache hit rates, output lengths, batch sizes, and draft configurations. The serving-cost analysis uses decoding and lookup costs measured on two H100 GPUs. Full-decoding cost is affine in output length $L$, with leave-one-length-out $R^2 = 0.9999$. Cost and throughput are computed from these measurements.

The lookahead is decoded only when there is a cache-hit candidate to check: a query that misses the cache is decoded exactly as without any cache and pays no lookahead cost. Let $h$ denote the query-level hit rate, $a$ the fraction of query hits admitted by the lookahead, $D(L)$ the full-decoding cost, $c_e$ the embedding and lookup cost, and $c_{\mathrm{look}}$ the cost of a $k$-token lookahead. Here $h$ is the rate at which the lookahead is triggered; the fraction of responses served from cache is $ha$. GPU-ms measures GPU time allocated per query, including all GPUs and dividing batch cost by the number of queries. Writing \emph{plain} for an undefended cache and excluding reuse of lookahead tokens, the mean costs per query are
\begin{align}
C_{\mathrm{no\mbox{-}cache}} &= D(L), \\
C_{\mathrm{plain}} &= c_e + (1-h)D(L), \\
C_{\mathrm{LaCache}} &= c_e + h c_{\mathrm{look}} + (1-ha)D(L).
\label{eq:serving_cost}
\end{align}
Thus, the lookahead cost per query is $h c_{\mathrm{look}}$. The $68.6$\,ms in Table\mref{tab:specdec} is the latency of one lookahead check, including drafting and verification. It is incurred only on the query-hit branch and is distinct from the batch-normalized GPU cost below.

When $aD(L) > c_{\mathrm{look}}$, the break-even query-hit rate relative to no cache is
\begin{equation}
\label{eq:breakeven}
h_{\min} = \frac{c_e}{aD(L)-c_{\mathrm{look}}},
\end{equation}
provided this rate is at most one. Unless otherwise stated, the following sweeps use $a = 1$ to isolate the cost of verification. We also report the break-even condition at $a = 0.928$, matching the default benign lookahead admission rate. Prefix-reuse credit is disabled throughout, so rejected hits pay for both the lookahead and full decoding.

\subsection{Cache Hit Rate and Output Length}\label{app:hit_length}

\textbf{Cache Hit Rate.} Table\mref{tab:cost_hits} varies $h$ at $k = 20$, $L = 256$, and batch size 32. \sys adds a marginal overhead to an undefended cache, and the overhead shrinks as hits get sparser: the decoding benefit per hit is roughly seven times the lookahead cost per hit, so a low $h$ decreases both. Across the measured range, \sys remains cheaper than decoding without a cache.

\begin{table}[!htbp]
\centering\small
\caption{Cost versus query-hit rate at $k = 20$, $L = 256$, batch size 32, and $a = 1$. Costs are GPU-ms per query.}
\label{tab:cost_hits}
\begin{tabular}{cccccc}
\toprule
$h$ & No cache & Plain cache & \sys & Lookahead cost & Overhead vs.\ plain cache \\
\midrule
0.05 & 142.9 & 136.5 & 137.5 & 0.96 & $+0.7\%$ \\
0.10 & 142.9 & 129.4 & 131.3 & 1.92 & $+1.5\%$ \\
0.20 & 142.9 & 115.1 & 118.9 & 3.83 & $+3.3\%$ \\
0.30 & 142.9 & 100.8 & 106.5 & 5.75 & $+5.7\%$ \\
\bottomrule
\end{tabular}
\end{table}

\textbf{Output Length.} Table\mref{tab:cost_length} fixes $h = 0.1$ and varies $L$ at batch size 32. The lookahead cost remains $1.92$ GPU-ms per query, so its relative cost decreases as output length grows. The break-even hit rate falls steeply with $L$, reaching $1.3\%$ at $L = 128$ for $a = 1$. At the fixed $h = 0.1$, the only evaluated length for which \sys costs more than no cache is $L = 32$, where every response is assumed to be only 32 tokens long. The final row repeats the calculation at $a = 0.928$, accounting for rejected hits that pay for lookahead and still require decoding; the break-even rate at $L = 128$ remains only $1.5\%$.

\begin{table}[!htbp]
\centering\small
\caption{Cost versus output length at $k = 20$, $h = 0.1$, and batch size 32. Cost rows use $a = 1$ and report GPU-ms per query. The last two rows give the minimum query-hit rate needed to match no-cache cost.}
\label{tab:cost_length}
\begin{tabular}{lcccccc}
\toprule
$L$ in tokens & 32 & 64 & 128 & 256 & 512 & 1024 \\
\midrule
No cache & 22.3 & 39.5 & 74.0 & 142.9 & 280.9 & 556.7 \\
Plain cache & 20.8 & 36.3 & 67.3 & 129.4 & 253.5 & 501.7 \\
\sys & 22.7 & 38.2 & 69.2 & 131.3 & 255.4 & 503.7 \\
\midrule
$h_{\min}$ at $a = 1$ & $23.5\%$ & $3.6\%$ & $1.3\%$ & $0.6\%$ & $0.3\%$ & $0.1\%$ \\
$h_{\min}$ at $a = 0.928$ & $48.7\%$ & $4.2\%$ & $1.5\%$ & $0.6\%$ & $0.3\%$ & $0.1\%$ \\
\bottomrule
\end{tabular}
\end{table}

\subsection{Batching and Throughput}\label{app:throughput}

\textbf{Batch Size.} Batching amortizes the lookahead, so \sys's advantage over no cache is stable across the batch range. Table\mref{tab:cost_batch} uses $k = 20$, $h = 0.5$, and $L = 256$, with decoding and lookup costs measured separately for each batch size. The lookahead stays below $14\%$ of \sys's cost even at this high query-hit rate, and the speedup relative to no cache ranges from $1.70\times$ to $1.86\times$.

\begin{table}[!htbp]
\centering\small
\caption{Effect of batching at $k = 20$, $h = 0.5$, $L = 256$, and $a = 1$. Cost rows report GPU-ms per query.}
\label{tab:cost_batch}
\begin{tabular}{lcccc}
\toprule
Batch size & 1 & 8 & 32 & 64 \\
\midrule
No cache & 2665 & 387 & 143 & 89 \\
Plain cache & 1341 & 195 & 72 & 45 \\
\sys & 1433 & 213 & 82 & 52 \\
Speedup vs.\ no cache & $1.86\times$ & $1.82\times$ & $1.75\times$ & $1.70\times$ \\
Lookahead share of \sys cost & $6.4\%$ & $8.2\%$ & $11.7\%$ & $13.7\%$ \\
\bottomrule
\end{tabular}
\end{table}

\textbf{Representative Operating Points.} Table\mref{tab:cost_throughput} reports serving-cost and throughput estimates at four operating points with co-varying $h$ and $L$. The lookahead cost is marginal relative to an undefended cache's savings: \sys keeps $73\%$ to $97\%$ of the cache's savings and remains $1.27\times$ to $7.56\times$ faster than no cache. Benefit retention is defined as
\begin{equation}
\mathrm{Retention} = \frac{C_{\mathrm{no\mbox{-}cache}} - C_{\mathrm{LaCache}}}{C_{\mathrm{no\mbox{-}cache}} - C_{\mathrm{plain}}}.
\end{equation}
A rejection is moreover not wasted work: the tokens verification already accepted are a valid prefix of that query's own response and can be reused in subsequent decoding. The estimates do not include this credit.

\begin{table}[!htbp]\small
\centering\small
\caption{Serving-cost and throughput estimates at $k = 20$, batch size 32, and $a = 1$. Cost rows report GPU-seconds per 1,000 queries; throughput is normalized per GPU-second.}
\label{tab:cost_throughput}\label{tab:serving_summary}
\footnotesize
\begin{tabular}{lcccc}
\toprule
Query-hit rate $h$ & 0.3 & 0.5 & 0.7 & 0.9 \\
Output length $L$ & 128 & 256 & 512 & 1024 \\
\midrule
No cache & 74.0 & 143 & 281 & 557 \\
Plain cache & 52.5 & 72.2 & 85.0 & 56.4 \\
\sys & 58.3 & 81.8 & 98.4 & 73.6 \\
\sys queries per GPU-second & 17.2 & 12.2 & 10.2 & 13.6 \\
Speedup vs.\ no cache & $1.27\times$ & $1.75\times$ & $2.85\times$ & $7.56\times$ \\
Retention of plain-cache benefit & $73\%$ & $86\%$ & $93\%$ & $97\%$ \\
\bottomrule
\end{tabular}
\end{table}

\subsection{Draft Model Configuration}\label{app:draft_configuration}

The reported lookahead latencies include full speculative-decoding wall-clock time, including drafting and verification. Table\mref{tab:draft_configuration} compares target-only decoding, the bundled MTP head, and a co-located draft model, all on the target's GPUs.

\begin{table}[!htbp]
\centering\small
\caption{Draft configuration and $k = 20$ lookahead latency. All configurations run on the target's GPUs.}
\label{tab:draft_configuration}
\begin{tabular}{lccc}
\toprule
Draft configuration & Extra weights & Latency in ms & Speedup vs.\ no draft \\
\midrule
None & 0 & 108.5 & $1.00\times$ \\
\texttt{Qwen3.6} MTP head & 0 GB & 68.6 & $1.6\times$ \\
\texttt{Qwen3.5-0.8B} draft & $\sim1.5$ GB & 50.0 & $2.2\times$ \\
\bottomrule
\end{tabular}
\end{table}

By default, \sys requires no separate GPU. The MTP head ships inside the \texttt{Qwen3.6} checkpoint and runs in the target's own forward pass, adding no extra weights, requiring no separate scheduler, and taking no batching capacity from the target. A co-located but separate draft model is faster still, but takes memory out of the target's budget and can lower its feasible batch size. A dedicated-GPU draft pays a cross-device transfer cost per lookahead, which is not worthwhile at $k = 20$. The speculative path does have higher latency variance as the acceptance rate fluctuates, but that affects tail latency rather than mean cost.

\FloatBarrier
\section{Reproducibility}\label{app:repro}

The artifacts (code, datasets, and configurations) are released at \url{https://github.com/JACKPURCELL/LaCache}.